\documentclass[11pt]{article}

\usepackage[truedimen,margin=25truemm]{geometry}

\usepackage{graphicx} 
\usepackage{epsfig} 
\usepackage{subfigure} 

\usepackage{amsmath}
\usepackage{amsthm}
\usepackage{amssymb}

\usepackage{algorithm}
\usepackage{algorithmic}

\newtheorem{definition}{Definition}

\usepackage{color}

\usepackage{ulem}

\usepackage{booktabs}

\usepackage{hyperref}

\usepackage{graphicx} 
\usepackage{epsfig}
\usepackage{subfigure}
\usepackage{bm}

\usepackage{amssymb}
\usepackage{amsmath}
\usepackage{amsthm}
\usepackage{mathtools}
\usepackage[raggedright]{titlesec} 
\usepackage{stmaryrd}
\usepackage{ulem}

\newtheorem{lemma}{Lemma}
\newtheorem{proposition}{Proposition}

\def\vec#1{\mbox{\boldmath $#1$}}
\def\mat#1{\mbox{\bf #1}}

\newcommand{\paratitle}[1]{\noindent {\bf #1}}

\title{Anchor Space Optimal Transport as a Fast Solution to Multiple Optimal Transport Problems}

\author{Jianming Huang \thanks{Graduate School of Fundamental Science and Engineering, WASEDA University, 3-4-1 Okubo, Shinjuku-ku, Tokyo 169-8555, Japan (e-mail: koukenmei@toki.waseda.jp) } \and Xun Su \thanks{Graduate School of Fundamental Science and Engineering, WASEDA University, 3-4-1 Okubo, Shinjuku-ku, Tokyo 169-8555, Japan (e-mail: suxun\_opt@asagi.waseda.jp)} \and Zhongxi Fang \thanks{Graduate School of Fundamental Science and Engineering, WASEDA University, 3-4-1 Okubo, Shinjuku-ku, Tokyo 169-8555, Japan (e-mail: fzx@akane.waseda.jp)} \and Hiroyuki Kasai \thanks{Department of Computer Science and Communication Engineering, WASEDA University, 3-4-1 Okubo, Shinjuku-ku, Tokyo 169-8555, Japan (e-mail: hiroyuki.kasai@waseda.jp)}}

\begin{document}

\maketitle

\begin{abstract}
In machine learning, Optimal Transport (OT) theory is extensively utilized to compare probability distributions across various applications, such as graph data represented by node distributions and image data represented by pixel distributions. In practical scenarios, it is often necessary to solve multiple OT problems. Traditionally, these problems are treated independently, with each OT problem being solved sequentially. However, the computational complexity required to solve a single OT problem is already substantial, making the resolution of multiple OT problems even more challenging. Although many applications of fast solutions to OT are based on the premise of a single OT problem with arbitrary distributions, few efforts handle such multiple OT problems with multiple distributions. Therefore, we propose the anchor space optimal transport (ASOT) problem: an approximate OT problem designed for multiple OT problems. This proposal stems from our finding that in many tasks the mass transport tends to be concentrated in a reduced space from the original feature space. By restricting the mass transport to a learned anchor point space, ASOT avoids pairwise instantiations of cost matrices for multiple OT problems and simplifies the problems by canceling {insignificant} transports. This simplification greatly reduces its computational costs. We then prove the upper bounds of its $1$-Wasserstein distance error between the proposed ASOT and the original OT problem under different conditions. Building upon this accomplishment, we propose three methods to learn anchor spaces {for reducing the approximation error}. Furthermore, our proposed methods\footnote{This paper is published in IEEE Transactions on Neural Networks and Learning Systems: \url{https://ieeexplore.ieee.org/document/10704726} \cite{huang2024anchor}} present great advantages for handling distributions of different sizes with GPU parallelization. The source code is available at \url{https://github.com/AkiraJM/ASOT}.
\end{abstract}

\section{Introduction}
As an important tool for comparing probability distributions, optimal transport (OT) theory \cite{Villani_2008_OTBook_s,Peyre_2019_OTBook_s} provides a robust and relevant distance for measuring two probability distributions on an already defined metric space: the so-called Wasserstein distance. However, because of the high computational complexity of $O(n^3\log(n))$ when solving a $n$-dimensional OT problem, it is difficult to apply it to practical machine learning tasks. In recent years, by virtue of computational improvements in solving an OT problem with entropic regularizers, computing an approximate Wasserstein distance with a negligible error is brought down to a quadratic time complexity. Therefore, many research works with applications of the OT theory showed up in many domains of machine learning \cite{arjovsky2017wasserstein,genevay2018learning}, which introduced a mini-batch OT loss to the generative models. In the graph learning domain, a framework of the Wasserstein kernel was proposed, which uses the Wasserstein distance to define distances or embeddings between graph-structured data \cite{togninalli2019wasserstein,kolouri2020wasserstein,Huang_SigPro_2020,Fang_AAAI_2023}. In the sequence learning domain, \cite{su2017order} modified the original OT problem to generate a solution with the characteristic of sequential data by adding temporal regularization terms. To compute solutions to OT problems more efficiently, many efforts have also been made. In this context, several related works include the well-known Sinkhorn solver \cite{Peyre_2019_OTBook_s}, which computes approximate solutions using entropic regularization. Additionally, the low-rank Sinkhorn \cite{scetbon2021low} solver calculates low-rank solutions for couplings based on Sinkhorn's solver. Moreover, the sliced Wasserstein distance \cite{bonneel2015sliced,kolouri2019generalized} presents a novel means of approximating the solution by solving one-dimensional OT problems on projections of data. {In addition to the original OT problem, which is limited to measures with equal total mass, a relaxed problem known as the unbalanced optimal transport (UOT) \cite{benamou2003numerical, Fukunaga_ICASSP_2022, Su_IJCNN_2024a} has been introduced for measures with differing total masses.}

Recently in machine learning tasks, one frequently encounters situations requiring the solution of multiple OT problems, as described in {\bf Definition \ref{Def.MOTS}} of Section \ref{Sec.Pre}. For example, in a graph classification based on the Wasserstein graph kernel \cite{togninalli2019wasserstein}, one can compute a pairwise Wasserstein graph distance matrix to instantiate a precomputed kernel for the support vector machine. The pairwise Wasserstein graph distance matrix includes distances of every pair of graphs of the dataset, each of which is computed by solving a single OT problem. Other similar examples include the $k$-NN sequence classification task based on the order-preserving Wasserstein distance \cite{su2017order}, {and the Wasserstein $k$-means clustering tasks \cite{Fukunaga_ICPR_2020}}. 
The previously described related works, as well as many other relevant works, are usually based on the premise of a single OT problem with two arbitrary distributions, without consideration of a scenario in which multiple OT problems with a given set of distributions are solved. In this case, as shown in the left part of Figure \ref{Fig.overall}, conventional one-by-one computation must solve each single OT problem separately {with their own cost matrices ($\mat{C}$ in Eq. (\ref{Eq.OT}))}. This mode of computation requires us to instantiate a cost matrix for every distribution pair, which is inefficient.

When considering these difficulties described above, it is apparent that in many machine learning tasks, the samples tend to be clustered in the feature space. Therefore, {the transportation of mass over some feature pairs is more important than others in the context of multiple OT problems. Therefore, we believe that transporting mass over a partial feature space is a more efficient approach for approximating solutions to multiple OT problems}. More concretely, let $\mathcal{X}$ be the feature space. For data that are not uniformly distributed in $\mathcal{X}$, the mass transport tends to concentrate in a reduced inner product space {$\mathcal{W}\times\mathcal{W}$} instead of $\mathcal{X}\times\mathcal{X}$, where {$\mathcal{W}$} is a subset of $\mathcal{X}$. Therefore, we {consider this $\mathcal{W}$} as a so-called {\it anchor space} and restrict the mass transport of the original OT problems to {$\mathcal{W}\times\mathcal{W}$}. Then we transform the original OT problem into {another} {\it approximate} OT problem with a set of {\it anchor points} {that can be considered as a discrete representation of {$\mathcal{W}$}. We name that transformed problem the {\it anchor space optimal transport} (ASOT) problem. As shown in Figure \ref{Fig.overall}, {we only design a single anchor space $\mathcal{W}$ for multiple OT problems so that all the transformed ASOT problems can share it}. In this manner, we need not instantiate the cost matrices for each, which helps to reduce computational time and memory space. Furthermore, with an anchor space excluding the {insignificant pairs}, we can further save time and memory at the expense of some approximate accuracy. We prove the upper bound of {the} absolute $1$-Wasserstein distance error of our proposed ASOT problem to the original OT problem. {In order to design an ASOT problem with a low $1$-Wasserstein distance error,} we propose a metric learning framework (ASOT-ML) to learn the anchor space {using this theoretical upper bound}. However, we found that ASOT-ML has several limitations which make it difficult to use under some conditions. Therefore, we further propose a lightweight version based on $k$-means clustering (ASOT-$k$) and a deep dictionary learning version (ASOT-DL).

Also, we find another attractive advantage of our proposed ASOT problem: Its convenience for parallel processing of multiple OT problems with size-variable discrete distributions. The Sinkhorn solver \cite{Peyre_2019_OTBook_s} provides a GPU-parallelizable solution to accelerate the OT batch processing. However, GPU parallelization with Sinkhorn's solver is only applicable to distributions with fixed sizes. It fails when dealing with {size-variable} data represented as discrete distributions of varying sizes, such as graph data, because the size-variable data cannot be stacked along the new batch dimensions. Although this difficulty can be resolved by an improved Sinkhorn solver with a block-diagonal stacking strategy \cite{fey2019fast} as described in Section \ref{Sec.Pre}, it still requires extra operations for block-diagonal stacking operations. Most importantly, it is unable to avoid the huge time and space costs imposed by the cost matrix instantiation. In our proposed ASOT, because the distributions of the transformed OT problems always have a fixed size, irrespective of whether they were originally size-variable, it can be directly applied to the {GPU-parallelized} Sinkhorn solver without the need {for} the block-diagonal stacking. 

Although similar works are related to OT problems with anchor spaces or subspaces \cite{lin2021making, paty2019subspace}, these specifically examine enhancing OT's robustness rather than considering multiple OT problems. For example, the latent OT \cite{lin2021making} divides the OT problem into three subproblems of OT: source space to source anchor space, source anchor space to target anchor space, and target anchor space to target space, which costs much greater computational time than OT. Conversely, our proposed ASOT problem includes only one OT problem on the anchor space. For that reason, it is solvable much faster.

Our contributions are threefold:
(i) We propose the anchor space optimal transport (ASOT) problem, which transforms multiple OT problems into approximate OT problems, benefiting from the lack of a need for pairwise cost instantiation and the always fixed distribution size. With the proposed ASOT problem, the computational cost of multiple OT problem solutions is greatly reduced. Furthermore, in the case of discrete distributions with different sizes, our proposed ASOT is solvable much faster with GPU parallelization.
(ii) We theoretically derive and prove the upper bound of absolute $1$-Wasserstein distance error between our ASOT and OT, which makes our approximation errors predictable. Moreover, to {design ASOT problems with low approximation error to the original problem}, we propose three methods for anchor space learning (ASOT-ML, ASOT-$k$, and ASOT-DL) for dealing with different situations.
(iii) We evaluate distance approximation errors and computational time by conducting numerical experiments on several widely used real-world datasets. 
These results demonstrate that our proposed approaches achieve a great reduction of computational time with reasonable approximation performances.

We divide the remainder of this paper into the following sections: (i) Section \ref{Sec.Pre} introduces the mathematical notation used and some important related works. (ii) Section \ref{Sec.SOT} introduces the definition and properties of our proposed ASOT problem. 
This section also introduces the ASOT batch processing method and its related complexity analysis. (iii) Section \ref{Sec.DeepDicLearn} introduces three methods for learning anchor spaces, which are used to construct the ASOT problem which minimizes the absolute $1$-Wasserstein distance error. (iv) Section \ref{Sec.exp} presents the setup and results of our evaluation experiment. 

\begin{figure}
\centering
\includegraphics[width= \hsize]{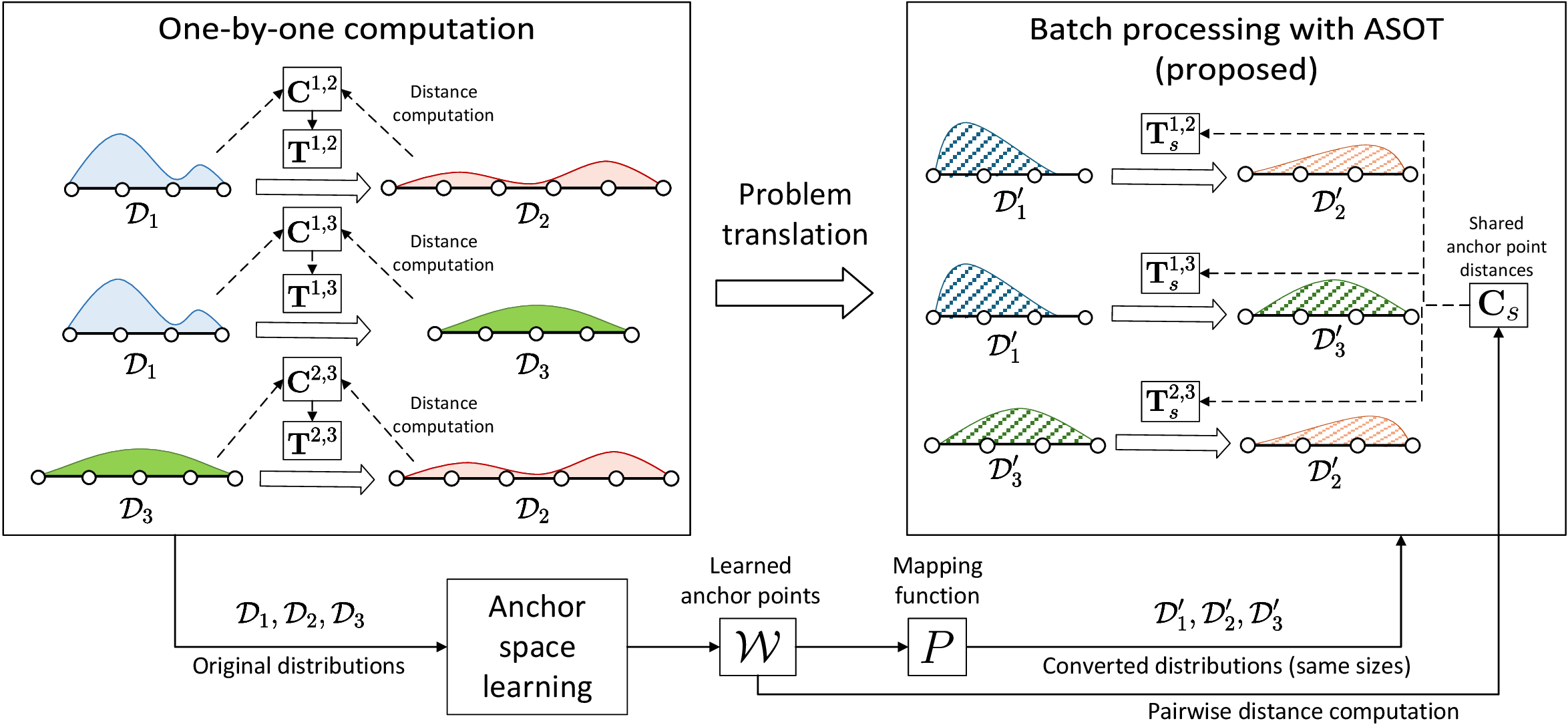}
\caption{{Illustration of the two strategies used to solve multiple OT problems: one-by-one computation {with the original OT problem} (left) and batch processing with the proposed ASOT problem (right). Assuming a multiple OT problem set $\mathcal{M}(\{\mathcal{D}_1, \mathcal{D}_2, \mathcal{D}_3\},\{(1,2),(1,3),(2,3)\})$ (refer to {\bf Definition \ref{Def.MOTS}}), then let $\mat{T}^{i,j}$ denote the transportation plan from $\mathcal{D}_i$ to $\mathcal{D}_j$. To solve the multiple problems, conventional one-by-one computation needs instantiation of the cost matrices ($\mat{C}^{1,2},\mat{C}^{1,3},\mat{C}^{2,3}$) for each distribution pair. Our proposed method learns an anchor space for all distributions. Thereafter, the original distributions are converted to new ones on} the learned anchor space as $\mathcal{D}_1^\prime, \mathcal{D}_2^\prime, \mathcal{D}_3^\prime$, which is a probability simplex based on a learnable anchor point set $\mathcal{W}$. When conducting batch processing with ASOT, all converted OT problems share the same cost matrix $\mat{C}_s$, which is a pairwise distance matrix of the anchor points in $\mathcal{W}$. Therefore, there is no need to instantiate cost matrices for each distribution pair.}
\label{Fig.overall}
\end{figure}

\section{Preliminaries}
\label{Sec.Pre}
Herein, we use lowercase letters to represent scalars such as $a,b$, and $c$. Lowercase letters with bold typeface such as $\vec{a}, \vec{b},$ {and} $\vec{c}$ represent vectors. Uppercase letters with bold typeface such as $\mat{A}, \mat{B},$ {and} $\mat{C}$ represent the matrices. For slicing matrices, we represent the element of the $i$-th row and $j$-th columns of $\mat{A}$ as $\mat{A}(i,j)$. Furthermore, $\mat{A}(i,:)$ denotes the vector of the $i$-th row of $\mat{A}$. Similarly, $\mat{A}(:,j)$ signifies the vector of $j$-th column. Also, $\vec{a}(i)$ denotes the $i$-th element of the vector $\vec{a}$.
Uppercase letters with an italic typeface denote a set, such as $\mathcal{A},\mathcal{B}, \mathcal{C}$. Also, $|\mathcal{A}|$ denotes the size of $\mathcal{A}$. $\mathbb{R}$ stands for the real number set. $\mathbb{R}_+$ represents a non-negative real number set. $\vec{1}_n$ denotes an all-one vector. $\langle\mat{A},\mat{B}\rangle := \sum_{i,j}\mat{A}(i,j)\mat{B}(i,j)$ signifies the Frobenius dot product. $\Delta^p_{n}:=\{\vec{a} \in \mathbb{R}^n : \vec{a}(i) \geq 0, \sum_{i=1}^n \vec{a}(i)=1\}$ denotes the $n$-dimensional probability simplex. Also, $\Delta^u_{n}:=\{\vec{a} \in \mathbb{R}^n : \vec{a}(i) \geq 0, \sum_{i=1}^n \vec{a}(i) \leq 1\}$ denotes the $n$-dimensional unit simplex. {$\llbracket N\rrbracket$} expresses the integer set $\{1,2,\cdots,N\}$. By default, we use element-wise division to divide vectors and use the element-wise exponential function $\exp$. {We use $\|\vec{a}\|_p$ to represent the $p$-norms of vector $\vec{a}$. In addition, $\|\vec{a}\|_\infty$ denotes the supremum norm of vector $\vec{a}$.} {All matrix norms are identical to those in the vector space, e.g., $\| \mat{A}\|_1=\|{\rm vec}(\mat{A})\|_1$, where ${\rm vec}(\mat{A}):\mathbb{R}^{m \times n} \rightarrow \mathbb{R}^{mn}$ represents a vectorization operator of a matrix to a column vector.} The following paragraphs outline the foundational knowledge pertinent to our research.

\paratitle{Optimal transport (OT) problem}. The OT problem \cite{Villani_2008_OTBook_s,Peyre_2019_OTBook_s} is intended to find a {transport} plan between two probability measures defined on a given metric space. For the main topic discussed herein, we only consider the OT problem based on discrete probability distributions, which could be viewed as finite sets of sampling points with feature vectors. Assume that two discrete probability distributions on different simplexes are given as $\vec{a}\in\Delta_{n}^p$ and $\vec{b}\in\Delta_{m}^p$. Furthermore, assume that the sampling points of $\vec{a}$ {and} $\vec{b}$ are on the same metric space, where the distance between the sampling points can be measured using a given distance metric. Let $\mat{C}\in\mathbb{R}^{n\times m}$ be the ground cost matrix, where $\mat{C}(i,j)$ denotes the distance measured between the $i$-th sampling point of $\vec{a}$ and $j$-th one of $\vec{b}$. The OT problem between $\vec{a}$ and $\vec{b}$ is defined as
\begin{equation}
	\mathop{\rm minimize}_{{\bf T}\in\mathcal{U}(\bm{a}, \bm{b})}~\langle\mat{T}, \mat{C}\rangle,\label{Eq.OT}
\end{equation}
where $\mathcal{U}(\vec{a}, \vec{b}) := \{\mat{T}\in\mathbb{R}_{+}^{n\times m}:\mat{T}\vec{1}_m = \vec{a}, \mat{T}^T\vec{1}_n = \vec{b}\}$ represents the set of all transportation plans between $\vec{a}$ and $\vec{b}$. The OT problem derives a new distance measuring the optimal transportation cost, designated as the Wasserstein distance. The $p$-Wasserstein distance between two discrete probability distributions $\vec{a},\vec{b}$ is defined as $W_p(\vec{a}, \vec{b},\mat{C}) := \mathop{\rm minimize}_{\bf{T}\in\mathcal{U}(\bm{a}, \bm{b})}~\langle\mat{T}, \mat{C}^p\rangle^\frac1p$. A wide application of the Wasserstein distance in specific machine learning tasks is to measure the $1$-Wasserstein distance between two sample sets. Given two sets of $d$-dimensional samples $\mathcal{X} := \{\vec{x}_i\}_{i=1}^{n}$ and $\mathcal{Y} := \{\vec{y}_j\}_{j=1}^{m}$ with $|\mathcal{X}| = n, |\mathcal{Y}| = m, \vec{x}_i,\vec{y}_j\in\mathbb{R}^d$ and their probability (mass) vector $\vec{a}\in\Delta_{n}^p$ and $\vec{b}\in\Delta_{m}^p$, respectively, then let $\mathcal{D}_x := (\mathcal{X}, \vec{a})$ and $\mathcal{D}_y := (\mathcal{Y}, \vec{b})$ denote the two distributions based on these two sample sets. To measure the optimal {transport} cost between {$\mathcal{D}_x$ and $\mathcal{D}_y$}, then the $1$-Wasserstein distance between {$\mathcal{D}_x$ and $\mathcal{D}_y$} is
\begin{equation}
	W_1(\mathcal{D}_x, \mathcal{D}_y,d_{\rm S}) := \mathop{\rm minimize}_{{\bf T}\in\mathcal{U}\left(\bm{a}, \bm{b}\right)}~\langle\mat{T}, \mat{C}\rangle,\label{Eq.Wass_app}
\end{equation}	
where $d_{\rm S}:\mathbb{R}^d\times\mathbb{R}^d\to \mathbb{R}_{+}$ represents the given ground cost metric
$\mat{C}(i,j) := d_{\rm S}(\vec{x}_i,\vec{y}_j)$. It is noteworthy that, because there is usually no prior information about the probability distributions that are already known in a practical situation, there is a common way to set the probability vectors as uniform distributions as $\vec{a} = \frac{\bm{1}_n}{n}, \vec{b} = \frac{\bm{1}_m}{m}$.

\paratitle{Sinkhorn's algorithm}. Because the OT problem expressed in Eq. (\ref{Eq.OT}) was able to be represented in the canonical form of a linear programming problem, which is a convex optimization problem, it is solvable in the complexity of $O(n^3\log(n))$ in the case of $n = m$. Another more popular scheme for solving OT problems is Sinkhorn's algorithm. Adding an entropic regularization term $H(\mat{T}) := -\sum_{i,j}\mat{T}(i,j)(\log(\mat{T}(i,j)) - 1)$ to Eq. (\ref{Eq.OT}) produces an entropically regularized OT (eOT) problem as
\begin{equation*}
	W^e_1(\mathcal{D}_x, \mathcal{D}_y):=\mathop{\rm minimize}_{{\bf T}\in\mathcal{U}(\bm{a}, \bm{b})}~\langle\mat{T}, \mat{C}\rangle - \varepsilon H(\mat{T}),
	\label{Eq.entro_Wass_app}
\end{equation*}
where $\varepsilon>0$ is a regularization parameter. This is solvable by introducing Sinkhorn's fixed-point iterations in the complexity of $O(n^2\log(n)\tau^{-3})$ with setting of $\varepsilon = \frac{4\log(n)}{\tau}$ \cite{Sinkhorn_PJM_1967,Cuturi_2013_NIPS,Benamou_2015_SIAMJSC_s,altschuler2017near}. In Sinkhorn's algorithm, the {transport} plan can be decomposed as $\mat{T} = \vec{u}^T\mat{K}{\vec{v}}$, where $\mat{K}:=\exp (-\frac{{\bf C}}\varepsilon )\in\mathbb{R}^{n\times m}$ is a Gibbs kernel computed from $\mat{C}$ and where $\vec{u}\in\mathbb{R}^n_{+}, \vec{v}\in\mathbb{R}^m_+$ are two scaling variables to be updated. Then the updating equations of the $j$-th iteration of Sinkhorn's algorithm are computed as $\vec{u}^{(j+1)} = \frac{{\bm a}}{{\bf K}{\bm v}^{(j)}}$ and ${\bm v}^{(j+1)} = \frac{{\bm b}}{{\bf K}^T{\bm u}^{(j+1)}}$.

\paratitle{Multiple optimal transport problem set}. We present our definition of the multiple OT problem set, which is the main application scenario of the proposal.
\begin{definition}[Multiple optimal transport problem set]
\label{Def.MOTS}
Assume a discrete distribution set of size $N$ as $\mathbb{D} := \{\mathcal{D}_1, \cdots, \mathcal{D}_N\}$. Given a pair set $\mathcal{P} \subset \{(i,j):i,j\in{\llbracket N\rrbracket},i\neq j\}$, the multiple OT problem set is defined as
\begin{equation}
\mathcal{M}(\mathbb{D},\mathcal{P}) := \left\{W_1(\mathcal{D}_i, \mathcal{D}_j, d_{\rm S}): (i,j)\in\mathcal{P}; \mathcal{D}_i,\mathcal{D}_j\in\mathbb{D}\right\},\label{Eq.MOTS}
\end{equation}
where $d_{\rm S}:\mathbb{R}^d\times\mathbb{R}^d\to\mathbb{R}_+$ represents the ground cost metric function. $d$ stands for the dimensions of each sample.
\end{definition}
\paratitle{Approaches to solve multiple OT problems}. {Solving problems of the multiple OT problem set can be viewed as solving OT problems of the given distribution pairs $\mathcal{P}$ over $\mathbb{D}$}. For cases in which distributions in $\mathbb{D}$ are of the same size, such problems are solvable in parallel using the GPU-parallelizable {Sinkhorn's algorithm}. When distribution sizes are allowed to vary, we cannot use the GPU-parallelizable {Sinkhorn's algorithm} directly and should instead use the block-diagonal stacking strategy.
Under the condition of size-fixed distributions, the updating equations of Sinkhorn's algorithm can be batched easily with stacking $\vec{u}, \vec{v}, \vec{a}, \vec{b}$ and $\mat{K}$ along a new batch dimension to create a volume. Although this strategy is inapplicable to size-variable distributions, we can adopt a block-diagonal stacking strategy where the kernel matrices $\mat{K}$ are stacked along the diagonal, while other vectors are stacked vertically instead of creating a new dimension. Thereafter, substituting the batched matrices and vectors into the original equation obtains a matrix of batched solutions. Such batch processing is also efficiently GPU-parallelizable with the release of dedicated CUDA kernels \cite{fey2019fast}.

\section{Anchor Space Optimal Transport (ASOT)}
\label{Sec.SOT}
This section introduces the definition of our proposed ASOT problem. We then present the theorem of an upper bound of absolute $1$-Wasserstein distance error between ASOT and normal OT, {followed by the complexity analysis and the entropic ASOT for GPU parallelization.}

\subsection{Motivation and definitions}
As described in the preceding sections, conventional one-by-one computation has two weak points that we particularly examine: (i) the need for cost matrices instantiation for every distribution pair and (ii) the need for extra operations to parallelize with size-variable distributions. To resolve these two difficulties, we approximately convert the multiple OT problems to problems on a reduced space based on several anchor sampling points. To remove the cost matrix instantiation described in (i), we design a so-called ASOT problem to approximate the original OT problem. It can be regarded as a new one with a low-rank representation of the original cost matrix. The most important point is that it has a fixed and shareable component ({\bf Lemma \ref{Lem.1}}). For the difficulty of (ii), in our proposed ASOT problem, mass transport is restricted to a learned anchor point set, which has a fixed size equal to the number of anchor points, thereby it does not require the operation of block-diagonal stacking. In summary, we propose to learn a so-called anchor space, over which our defined ASOT problems can be solved efficiently. Moreover, it has an upper bounded absolute $1$-Wasserstein distance error to the original OT problem. {Prior to defining the ASOT problem, we first introduce the concept of anchor space and a supporting simplex lemma. These foundational elements are essential for the formal definition of the ASOT problem}.

\begin{definition}[Anchor space]
\label{Def.AS}
{
Assume an original metric space on $\mathbb{R}^d$ with a metric function $d_{\rm S}:\mathbb{R}^d\times\mathbb{R}^d\to\mathbb{R}_+$, denoted as $S := (\mathbb{R}^d, d_{\rm S})$. Let $\mathcal{W}:=\{\vec{w}_i\}_{i=1}^k$ be a set of $k$ points on $S(\mathbb{R}^d, d_{\rm S})$, where $\vec{w}_i \in \mathbb{R}^d$. Given a new metric function $d_{\rm AS}:\mathbb{R}^d\times\mathbb{R}^d\to\mathbb{R}_+$, we define $A := (\mathcal{W}, d_{\rm AS})$ as an anchor space of $S$.
}
\end{definition}

\begin{lemma}[Simplex lemma]
\label{Lem.Sim}
{
Let $\{\vec{u}_i\}_{i=1}^n$ be a set of $n$ $d$-dimensional probability simplex vectors, where $\vec{u}_i\in\Delta^p_d$, and let $\vec{v}\in\Delta^p_n$ be an $n$-dimensional simplex vector. Define a new vector $\vec{w} := \sum_{i=1}^n \vec{u}_i \cdot \vec{v}(i)$. Then $\vec{w}\in\Delta^p_d$.
}
\end{lemma}
\begin{proof}
{
To prove that $\vec{w} \in \Delta^p_d$, we first establish that $\vec{w} \in \mathbb{R}^d_+$. This is evident because $\vec{w}(j) = \sum_{i=1}^n \vec{u}_i(j) \cdot \vec{v}(i)$, where $\vec{u}_i(j) \geq 0$ and $\vec{v}(i) \geq 0$. Let $\mat{U} \in \mathbb{R}^{n \times d}$ be the concatenated matrix of $\{\vec{u}_i\}_{i=1}^n$, with $\mat{U}(i,:) = \vec{u}_i$. Thus, we obtain $\vec{w} = \mat{U}^T \vec{v}$. Consequently, we have $\vec{w}^T \vec{1}_d = \vec{v}^T \mat{U} \vec{1}_d = \vec{v}^T \vec{1}_n = 1$. From these results, it follows that $\vec{w}(j) \geq 0$ and $\sum_{j=1}^d \vec{w}(j) = 1$, which completes the proof.
}
\end{proof}

\begin{definition}[Anchor space optimal transport (ASOT) problem]
\label{Def.ASOT}
Let $d_{\rm S}, d_{\rm AS}:\mathbb{R}^d\times\mathbb{R}^d\to\mathbb{R}_+$ be two metric functions on $\mathbb{R}^d$. Assume two discrete distributions $\mathcal{D}_x := (\mathcal{X}, \vec{a}), \mathcal{D}_y := (\mathcal{Y}, \vec{b})$ on a source metric space $S:=(\mathbb{R}^d, d_{\rm S})$, where $\mathcal{X} := \{\vec{x}_i\}_{i=1}^{n}$ and $\mathcal{Y}:=\{\vec{y}_j\}_{j=1}^{m}$ denote the sample sets. Also, $\vec{x}_i$ and $\vec{y}_j\in\mathbb{R}^d$ are $d$-dimensional feature vectors. $\vec{a} \in \Delta^p_n, \vec{b} \in \Delta^p_m$ respectively denote their probability (mass) vectors. {Given an anchor space $A:=(\mathcal{W},d_{\rm AS})$ and a mapping function $P:\mathbb{R}^{d}\to \Delta^p_k$, where $\mathcal{W}:=\{\vec{w}_i\}_{i=1}^k$ is a set of $k$ anchor points on $S$, and $\vec{w}_i \in \mathbb{R}^d$}, we define the $1$-Wasserstein distance of ASOT problem as
\begin{eqnarray}
	W_{\rm AS}(\mathcal{D}_x, \mathcal{D}_y, \mathcal{W}, P, d_{\rm AS}) \ :=\  \mathop{\rm minimize}_{{\bf T}_s\in\mathcal{U}(\bm{a}^\prime, \bm{b}^\prime)}~\langle\mat{T}_s, \mat{C}_s\rangle,\label{Eq.subspaceOT}\\
	\vec{a}^\prime := \sum_{i=1}^nP(\vec{x}_i)\vec{a}(i),\quad  \vec{b}^\prime := \sum_{j=1}^mP(\vec{y}_j)\vec{b}(j),\label{Eq.SOTcondi}
\end{eqnarray}
where $W_{\rm AS}(\mathcal{D}_x, \mathcal{D}_y, \mathcal{W}, P, d_{\rm AS})$ denotes the $1$-Wasserstein distance of the ASOT problem, and where $\mat{T}_s\in\mathbb{R}^{k\times k}$ denotes the anchor point transportation plan. $\mat{C}_s\in\mathbb{R}^{k\times k}$ denotes the anchor point cost matrix, where $\mat{C}_s(i,j) := d_{\rm AS}(\vec{w}_i, \vec{w}_j)$. {$\vec{a}^\prime, \vec{b}^\prime\in\Delta^p_k$ according to {\bf Lemma \ref{Lem.Sim}}.}
\end{definition}

We further define an entropic variant of the ASOT problem, envisioning the application of Sinkhorn's algorithm for GPU parallelization.
\begin{definition}[Entropic anchor space optimal transport (eASOT) problem]
\label{Def.eASOT}
Following the same definitions of {\bf Definition \ref{Def.ASOT}}, then an entropic variant problem of Eq. (\ref{Eq.subspaceOT}) is defined as
\begin{equation}
	W^e_{\rm AS}(\mathcal{D}_x, \mathcal{D}_y, \mathcal{W}, P, d_{\rm AS}) := \mathop{\rm minimize}_{{\bf T}_s\in\mathcal{U}(\bm{a}^\prime, \bm{b}^\prime)}~\langle\mat{T}_s, \mat{C}_s\rangle- \varepsilon H(\mat{T}_s)\label{Eq.entropic_subspaceOT},
\end{equation}
where $H(\mat{T}) := -\sum_{i,j}\mat{T}(i,j)(\log(\mat{T}(i,j)) - 1)$. This problem is called the entropic ASOT problem, expressed as eASOT.
\end{definition}

\subsection{{Properties of ASOT}}

This subsection presents several important and useful properties of our proposed ASOT and relative proofs.

\paratitle{ASOT meets the definition of an OT problem.} When reviewing the definition of an OT problem, it is readily apparent that ASOT is also an OT problem over two new distributions $\mathcal{D}_x^\prime := (\mathcal{W},\vec{a}^\prime)$ and $\mathcal{D}_y^\prime := (\mathcal{W},\vec{b}^\prime)$ with a ground cost matrix $\mat{C}_s$ computed using a distance measure $d_{\rm AS}$. This property enables us to solve an ASOT problem in the same way as OT does. The minimal cost of the eASOT problem is also guaranteed to have the same properties as the $1$-Wasserstein distance.

\paratitle{ASOT always has fixed distribution sizes.} As shown in the equations of {\bf Definition \ref{Def.ASOT}}, the sizes of distributions $\mathcal{D}_x^\prime,\mathcal{D}_y^\prime$, the transportation plan $\mat{T}_s$ and the cost matrix $\mat{C}_s$ are always fixed, which depends only on the {anchor space $A$}. Therefore, this property makes the ASOT problem efficiently solvable with GPU parallelization.

\paratitle{Upper bounded error to OT.} To convert an arbitrary OT problem into the form of our proposed ASOT problem while ensuring a controllable distance error, we present the upper bound of their absolute $1$-Wasserstein distance error as the following {\bf Proposition \ref{Prop.UB}}. Before that, we first prepare the following {\bf Lemma \ref{Lem.1}}, which will be used in the proof, and also present another important property of our proposed ASOT problem: its equivalence to a specific OT problem {with an expanded cost matrix}.

\begin{lemma}[Equivalence lemma]
		\label{Lem.1}
		{
		Assume an OT problem of two discrete distributions $\mathcal{D}_x := (\mathcal{X}, \vec{a}), \mathcal{D}_y := (\mathcal{Y}, \vec{b})$ as
		\begin{equation}
			\mathop{\rm minimize}_{{\bf T}\in\mathcal{U}(\bm{a}, \bm{b})}~\langle\mat{T}, \widehat{\mat{C}}\rangle,\label{Eq.el}
		\end{equation}
		where $\vec{a}\in\Delta^p_n, \vec{b}\in\Delta^p_m$, and $\widehat{\mat{C}}$ denotes the ground cost matrix.
		This problem is equivalent to the ASOT problem of $W_{\rm AS}(\mathcal{D}_x, \mathcal{D}_y, \mathcal{W}, P, d_{\rm AS})$ that satisfies $\widehat{\mat{C}} := \mat{Z}_x\mat{C}_s\mat{Z}_y^T$. $\mat{C}_s\in\mathbb{R}^{k\times k}$ denotes the anchor point cost matrix derived from $\mathcal{W}$. $\mat{Z}_x \in \mathbb{R}^{n\times k}, \mat{Z}_y \in \mathbb{R}^{m\times k}$, and $\mat{Z}_x(i,:) := P(\vec{x}_i), \mat{Z}_y(j,:) := P(\vec{y}_j)$.
		}
\end{lemma}

\begin{proof}
{
{We first define the minimal cost of Eq. (\ref{Eq.el}) as}
\begin{equation}
	W_1 := \mathop{\rm minimize}_{{\bf T}\in\mathcal{U}(\bm{a}, \bm{b})}~\langle \mat{T},\widehat{\mat{C}} \rangle\label{Eq.el2},
\end{equation}
where $\widehat{\mat{C}}\in\mathbb{R}^{n\times m}$ can be decomposed into $\widehat{\mat{C}} = \mat{Z}_x\mat{C}_s\mat{Z}^T_y$.
{Let $c(\mat{T}) := \langle\mat{T}, \widehat{\mat{C}}\rangle$ be the total cost w.r.t. $\mat{T}$}, then we have
\begin{eqnarray*}
c(\mat{T}) & = &\sum_{{i\in\llbracket n\rrbracket,j\in\llbracket m\rrbracket}} \sum_{{u,v\in\llbracket k\rrbracket}} \mat{Z}_x(i,u)\mat{Z}_y(j,v)\mat{C}_s(u,v)\mat{T}(i,j)\\
	    & = &\sum_{{u,v\in\llbracket k\rrbracket}} \sum_{{i\in\llbracket n\rrbracket,j\in\llbracket m\rrbracket}} \mat{Z}_x(i,u)\mat{Z}_y(j,v)\mat{T}(i,j)\mat{C}_s(u,v).
\end{eqnarray*}

Let $\mat{T}_s :=  \mat{Z}^T_x\mat{T}\mat{Z}_y\in\mathbb{R}^{k\times k}$. The above equations can be converted as $c(\mat{T})=\langle\mat{T}_s,\mat{C}_s \rangle$. Substituting it into Eq. (\ref{Eq.el2}) produces
\begin{equation}
W_1 =  \mathop{\rm minimize}_{{\bf T}\in\mathcal{U}(\bm{a}, \bm{b})}~\langle\mat{T}_s,\mat{C}_s \rangle.\label{Eq.beforeR}
\end{equation}

To reparameterize this problem with $\mat{T}_s$, we must find the domain of $\mat{T}_s$. Right multiplying $\mat{T}_s$ with $\vec{1}_k$ obtains
\begin{equation}
\mat{T}_s\vec{1}_k = \mat{Z}^T_x\mat{T}\mat{Z}_y\vec{1}_k.\label{Eq.domainT}
\end{equation}

Because each row of $\mat{Z}_y$ is a simplex vector, we have $\mat{Z}_y\vec{1}_k = \vec{1}_m$. Subtituting it into Eq. (\ref{Eq.domainT}) obtains $\mat{T}_s\vec{1}_k \quad = \quad \mat{Z}_x^T\mat{T}\vec{1}_m = \mat{Z}^T_x\vec{a}$.
In the same way, we can obtain {$\mat{T}_s^T\vec{1}_k = \mat{Z}^T_y\vec{b}$}. Futhermore, because for all ${i\in \llbracket n\rrbracket, j\in\llbracket m\rrbracket}$, we have $\mat{T}(i,j) \geq 0$, then $\forall {u,v\in\llbracket k\rrbracket}, \mat{T}_s(u,v) \geq 0$ holds. In summary, $\mat{T}_s\in\mathcal{U}(\vec{a}^\prime, \vec{b}^\prime)$, where
$\vec{a}^\prime = \mat{Z}^T_x\vec{a} = \sum_{i=1}^n\mat{Z}_x(i,:)\vec{a}(i)$,
$\vec{b}^\prime = \mat{Z}^T_y\vec{b} = \sum_{i=j}^m\mat{Z}_y(j,:)\vec{b}(j)$.
Finally, Eq. (\ref{Eq.beforeR}) is reparameterized as
\begin{equation}
W_1 =  \mathop{\rm minimize}_{{\bf T}_s\in\mathcal{U}(\bm{a}^\prime, \bm{b}^\prime)}\ \langle\mat{T}_s,\mat{C}_s \rangle.\label{Eq.afterR}
\end{equation}

Give $\mathcal{X}, \mathcal{Y}, \mathcal{W}$, mapping function $P$, and metric function $d_{\rm AS}$ as defined in {\bf Definition \ref{Def.AS}}, which satisfy $\forall {u,v \in \llbracket k\rrbracket}, \mat{C}_s(u,v) := d_{\rm AS}(\vec{w}_u, \vec{w}_v)$ and $\forall {i\in\llbracket n\rrbracket}, \mat{Z}_x(i,:) := P(\vec{x}_i)$ and $\forall {j\in\llbracket m\rrbracket}, \mat{Z}_y(j,:) := P(\vec{y}_j)$. Then Eq. (\ref{Eq.afterR}) is equivalent to the ASOT problem of $W_{\rm AS}(\mathcal{D}_x, \mathcal{D}_y, \mathcal{W},P,d_{\rm AS})$, which completes the proof.
}
\end{proof}

\begin{proposition}[Upper bound of absolute $1$-Wasserstein distance error]
\label{Prop.UB}
Assume an OT problem for $1$-Wasserstein distance $W_1(\mathcal{D}_x, \mathcal{D}_y,d_{\rm S})$, with $\mat{C}\in\mathbb{R}^{n\times m}$ as the ground cost matrix and $d_{\rm S}$ as the metric function. Define an ASOT problem $W_{\rm AS}(\mathcal{D}_x, \mathcal{D}_y, \mathcal{W}, P, d_{\rm AS})$ of \textbf{Definition \ref{Def.ASOT}} with given $\mathcal{W}, P$ and $d_{\rm AS}$. Let $\mat{Z}_x \in\mathbb{R}^{n\times k}, \mat{Z}_y\in\mathbb{R}^{m\times k}$ be matrices where $\mat{Z}_x(i,:) := P(\vec{x}_i)$ and $\mat{Z}_y(j,:) := P(\vec{y}_j)$. We define a reconstructed cost matrix $\widehat{\mat{C}} := \mat{Z}_x\mat{C}_s\mat{Z}_y^T \in \mathbb{R}^{n\times m}$. Then the following inequality holds:
\begin{equation}
|W_{\rm AS}(\mathcal{D}_x, \mathcal{D}_y, \mathcal{W}, P, d_{\rm AS}) - W_1(\mathcal{D}_x, \mathcal{D}_y,d_{\rm S})| \leq \|\widehat{\mat{C}} - \mat{C}\|_1.\label{Eq.UB}
\end{equation}
\end{proposition}

\begin{proof}
{
According to {\bf Lemma \ref{Lem.1}}, the ASOT problem of $W_{\rm AS}(\mathcal{D}_x, \mathcal{D}_y, \mathcal{W}, P, d_{\rm AS})$ is equivalant to an OT problem of Eq. (\ref{Eq.el}). Let $\widehat{\mat{T}}^\star$ be the optimal solution of Eq. (\ref{Eq.el}), and let $\mat{T}^\star$ be the optimal solution of Eq. (\ref{Eq.Wass_app}). We write the absolute $1$-Wasserstein distance error as
\begin{eqnarray*}
	\eta&:=&|W_{\rm AS}(\mathcal{D}_x, \mathcal{D}_y, \mathcal{W}, P, d_{\rm AS}) - W_1(\mathcal{D}_x, \mathcal{D}_y,d_{\rm S})|\\
	&=& \left|\langle \widehat{\mat{T}}^\star, \widehat{\mat{C}} \rangle - \langle \mat{T}^\star, \mat{C} \rangle\right|.
\end{eqnarray*}

In the case of $W_{\rm AS}(\mathcal{D}_x, \mathcal{D}_y, \mathcal{W}, P, d_{\rm AS}) > W_1(\mathcal{D}_x, \mathcal{D}_y,d_{\rm S})$, we have
\begin{eqnarray*}
	\eta&=&\sum_{{i\in\llbracket n\rrbracket,j\in\llbracket m\rrbracket}}\widehat{\mat{T}}^\star(i,j)\widehat{\mat{C}}(i,j) -\mat{T}^\star(i,j)\mat{C}(i,j)\\
	&\leq&\sum_{{i\in\llbracket n\rrbracket,j\in\llbracket m\rrbracket}}\mat{T}^\star(i,j)\widehat{\mat{C}}(i,j) - \mat{T}^\star(i,j)\mat{C}(i,j)\\
	&=&\sum_{{i\in\llbracket n\rrbracket,j\in\llbracket m\rrbracket}}\mat{T}^\star(i,j)\left(\widehat{\mat{C}}(i,j) - \mat{C}(i,j)\right).
\end{eqnarray*}

{The inequality {above} is because the OT problem is convex, leading to only one single global minimum}. Moreover, $\mat{T}^\star(i,j) \leq 1$ and $\sum_{{i\in\llbracket n\rrbracket,j\in\llbracket m\rrbracket}}\mat{T}^\star(i,j) = 1$ holds because $\mat{T}^\star \in\mathcal{U}(\vec{a}, \vec{b})$ and $\vec{a}\in\Delta^p_n, \vec{b}\in\Delta^p_m$. Then we have
\begin{eqnarray*}
	\eta&\leq&\sum_{{i\in\llbracket n\rrbracket,j\in\llbracket m\rrbracket}}\mat{T}^\star(i,j)\left(\widehat{\mat{C}}(i,j) - \mat{C}(i,j)\right)\\
	&\leq&\sum_{{i\in\llbracket n\rrbracket,j\in\llbracket m\rrbracket}}\mat{T}^\star(i,j)\left|\widehat{\mat{C}}(i,j) - \mat{C}(i,j)\right|\\
	&\leq&\|\widehat{\mat{C}} - \mat{C}\|_\infty \leq \|\widehat{\mat{C}} - \mat{C}\|_1.
\end{eqnarray*}

In the case of $W_{\rm AS}(\mathcal{D}_x, \mathcal{D}_y, \mathcal{W}, P, d_{\rm AS}) \leq W_1(\mathcal{D}_x, \mathcal{D}_y,d_{\rm S})$, similarly we have
\begin{eqnarray*}
	\eta&=&\sum_{{i\in\llbracket n\rrbracket,j\in\llbracket m\rrbracket}}\mat{T}^\star(i,j)\mat{C}(i,j) - \widehat{\mat{T}}^\star(i,j)\widehat{\mat{C}}(i,j) \\
	&\leq&\sum_{{i\in\llbracket n\rrbracket,j\in\llbracket m\rrbracket}}\widehat{\mat{T}}^\star(i,j)\left(\mat{C}(i,j) - \widehat{\mat{C}}(i,j)\right)\\
	&\leq&\sum_{{i\in\llbracket n\rrbracket,j\in\llbracket m\rrbracket}}\widehat{\mat{T}}^\star(i,j)\left|\mat{C}(i,j) - \widehat{\mat{C}}(i,j)\right|\\
	&\leq&\|\mat{C} - \widehat{\mat{C}}\|_\infty \leq \|\widehat{\mat{C}} - \mat{C}\|_1.
\end{eqnarray*}

Combining the both cases above completes the proof.
}
\end{proof}

Based on {\bf Proposition \ref{Prop.UB}}, we can convert a given OT problem into an ASOT problem by minimizing $\|\widehat{\mat{C}} - \mat{C}\|_1$, which can be viewed as a function of three variables: $\mathcal{W}, P$, and $d_{\rm AS}$. {It is noteworthy that although $\|\widehat{\mat{C}} - \mat{C}\|_\infty$ is a {lower} upper bound, we still use the $\ell_1$-norm because it is easier to solve}. To solve this optimization problem, we propose three methods for different backgrounds, one of which is based on the metric learning and the other two of which are based on $k$-means clustering and deep dictionary learning.

\subsection{Complexity analysis}
We present a comparison of the time and space complexity of the one-by-one computation {with original OT problem} and batch processing with our ASOT problem. For this comparison, we let $N$ denote the number of distributions, and let $s$ denote the average number of samples in a distribution. We assume full-pair OT problems on the distributions where OT problems are solved for all combinations of distributions.

{\paratitle{Complexities for the preparation process.} For time complexity, one-by-one computation requires $O(N^2s^2)$ computation to instantiate {cost matrices} of all combinations, whereas ASOT only requires $O(Ns)$ to convert distributions into anchor space distributions with Eq. (\ref{Eq.SOTcondi}). For space complexity, one-by-one computation requires $O(N^2s^2)$ memory space for restoring every cost matrix, whereas ASOT only requires $O(k^2)$ because we need only to store the pairwise anchor point cost matrix $\mat{C}_s\in\mathbb{R}^{k\times k}$. Therefore, by setting $k$ and $s$ at the same level of magnitude, ASOT will save much memory space and will be able to create a mini-batch {of OT problems} of a larger size {in parallelization process}, which can also accelerate computation given a limited memory size.}

{\paratitle{Complexities for the solution process.} Complexities of the solution process depend on which algorithm we use to solve the OT problem. We take the Sinkhorn solver as an example. One-by-one computation requires an $O(s^3)$ to compute $\mat{K}\vec{v}$ and $\mat{K}^T\vec{u}$ for each single Sinkhorn loop. By contrast, ASOT requires $O(k^3)$, where $k$ denotes the number of anchor points.

\subsection{GPU-parallelizable Sinkhorn for eASOT problem}
\label{Sec.GPU}
The preceding subsection introduced the definition of the ASOT problem on a single pair of distributions. For GPU parallelization on a multiple OT problem set, we consider computing their $1$-Wasserstein distances by solving the transformed eASOT problems over a shared anchor space of {$A$}. In this case, all the distribution pairs in the batch share a set of parameters of $\mathcal{W}, P$ and $d_{\rm AS}$, which need only be learned once.
According to {\bf Remark 4.16} in \cite{Peyre_2019_OTBook_s}, we use the GPU parallelization version of Sinkhorn's algorithm to solve the eASOT problem. 
More concretely, given $N$ {converted} probability distribution pairs of {ASOT} $\{(\vec{a}_i, \vec{b}_i)\}_{i=1}^N$, where $\vec{a}_i, \vec{b}_i \in\mathbb{R}^k$, and $\mat{A} := [\vec{a}_1, \cdots, \vec{a}_N]\in\mathbb{R}^{k\times N}$, $\mat{B}:= [\vec{b}_1, \cdots, \vec{b}_N]\in\mathbb{R}^{k\times N}$, the following equation updates the dual variables $\vec{u}_i \in \mathbb{R}^k$ and $\vec{v}_i \in\mathbb{R}^k$ at the $j$-th iteration of Sinkhorn's algorithm with GPU parallelization.
\begin{eqnarray*}
	\mat{U}^{(j+1)} = \frac{\mat{A}}{\mat{K}\mat{V}^{(j)}}, \quad \mat{V}^{(j+1)} = \frac{\mat{B}}{\mat{K}^T\mat{U}^{(j+1)}}.
\end{eqnarray*}

Therein, $\mat{U} := [\vec{u}_1, .\cdots, \vec{u}_N], \mat{V}:= [\vec{v}_1, .\cdots, \vec{v}_N]\in\mathbb{R}^{k\times N}$, and $\mat{K} := \exp\left(-\frac{{\bf C}_s}{\varepsilon}\right) \in \mathbb{R}^{k\times k}$ is the Gibbs kernel. 

\section{{Anchor Space Learning}}
\label{Sec.DeepDicLearn}

As described above, to convert OT problems into ASOT problems {with low approximation error}, we must learn an anchor space using the upper bound in {\bf Proposition \ref{Prop.UB}} w.r.t. $\mathcal{W},P$ and $d_{\rm AS}$. This section presents a deep metric learning framework, a $k$-means-based framework, and a deep dictionary learning framework to solve this {learning} problem.

\subsection{Deep metric learning ASOT (ASOT-ML)}
{In our metric learning framework for ASOT}, to minimize $\|\widehat{\mat{C}} - \mat{C}\|_1$, we first assume that $\widehat{\mat{C}}$ is computed using a parameterized metric function $f(\mathcal{W},P,d_{\rm AS}) $, defined as presented below.
\begin{eqnarray*}
 	\widehat{\mat{C}}(i,j) &:=& f(\mathcal{W},P,d_{\rm AS})\\
 	&:=& \sum_{{u,v\in\llbracket k\rrbracket}}\mat{Z}_x(i,u)\cdot\mat{Z}_y(j,v)\cdot d_{\rm AS}(\vec{w}_u, \vec{w}_v)\\
 	&=& \sum_{{u,v\in\llbracket k\rrbracket}}P(\vec{x}_i)(u)\cdot P(\vec{y}_j)(v)\cdot d_{\rm AS}(\vec{w}_u, \vec{w}_v).
\end{eqnarray*}

Then the problem becomes a metric learning problem from $f$ to $d_{\rm S}$, which is the ground metric function of $\mat{C}$. Because we have two variables $P$ and $d_{\rm AS}$ that are functions, we parameterize $P$ with a multilayer perception network (MLP). We use the Mahalanobis distance for $d_{\rm AS}$, which is computed as $d_{\rm AS}(\vec{w}_u, \vec{w}_v) := \|\mat{M}\vec{w}_u - \mat{M}\vec{w}_v\|_2$, where $\mat{M}\in\mathbb{R}^{h\times k}$ is a learned parameter matrix. Also, $h$ denotes the number of hidden dimensions. It is noteworthy that, to ensure that the outputs of $P$ are within the probability simplex, the outputs of the MLP for $P$ are normalized by dividing their $\ell_1$-norms. Thereafter, we adopt an automatic differentiation neural network system to build our framework, as shown in (a) in Figure \ref{Fig.DicLearn}. The loss function for the framework is defined as $\eta_{i,j} :=|f(\mathcal{W},P,d_{\rm AS}) - d_{\rm S}(\vec{x}_i, \vec{y}_j )|$ for each pair of sampling points. To avoid the subgradient, we adopt an $\ell_2$-loss as $\mathcal{L}_{\rm ML} := \frac{C}{nm}\sum_{{i\in \llbracket n\rrbracket, j\in \llbracket m\rrbracket}}\eta_{i,j}^2$, where $C$ is a scalar.

However, several limitations exist for ASOT-ML: (i) Quadratic complexity --- Because ASOT-ML needs pairwise distance errors, it has both time complexity and space complexity of $O(n^2)$. Therefore, it is impossible to conduct full pairwise distance computation on a large-scale dataset. One must use mini-batch pairwise distances, which might affect its performance. (ii) Unsafe solution --- ASOT-ML cannot control the scale of the learned cost matrix, which might cause a risk of overflow while computing Sinkhorn's algorithm. That is true because, when $\mat{C}_s$ has an overly large or overly small value, the computation of $\mat{K} := \exp(-\frac{{\bf C}_s}{\varepsilon})$ will produce outliers that affect the update process. {Although the log-domian Sinkhorn ({\bf Remark 4.23} in \cite{Peyre_2019_OTBook_s}) can resolve this difficulty, it also makes the computation hard to be parallelized.}

\subsection{$k$-means-based (ASOT-$k$) and Deep dictionary learning ASOT (ASOT-DL)}

\paratitle{Upper bound under the one-hot condition.} To avoid limitations of ASOT-ML, {we consider the ASOT problem under a one-hot condition. According to \textbf{Proposition \ref{Prop.UBoh}}, the metric learning problem for learning the anchor space can be reformulated as a feature reconstruction problem. This reformulation effectively reduces both computational complexity and the risk of unsafe solutions. Consequently,} we propose a lightweight {$k$-means framework and a} deep dictionary learning framework based on the ASOT under the one-hot condition that only {requires} $O(n)$ level complexity. We present the following proposition of the upper bound and its proof.
\begin{proposition}[Upper bound of absolute $1$-Wasserstein error under the one-hot condition]
\label{Prop.UBoh}
Let $\vec{q}_a^b\in \mathbb{R}^a$ denote an $a$-dimensional one-hot vector of which the $b$-th element is equal to $1$, i.e. $\vec{q}_a^b(b) := 1$. {Let $\mathcal{Q}:=\{\vec{q}_k^i\}_{i=1}^k$ be the $k$-dimensional one-hot vector set}. Let $P_{o}:\mathbb{R}^d\to\mathcal{Q}$ be a one-hot mapping function. Let $\mathcal{W} := \{\vec{w}_i\}_{i=1}^k$ and $\mat{W} := [\vec{w}_1,\cdots,\vec{w}_k] \in\mathbb{R}^{d\times k}, \vec{w}_i\in\mathbb{R}^d$. Also, let there be two discrete distributions $\mathcal{D}_x := (\mathcal{X}, \vec{a})$ and $\mathcal{D}_y := (\mathcal{Y}, \vec{b})$, where $\mathcal{X} := \{\vec{x}_i\}_{i=1}^{n}, \mathcal{Y}:=\{\vec{y}_j\}_{j=1}^{m}$ and $\vec{a}\in\Delta^p_n, \vec{b}\in\Delta^p_m$. $\vec{x}_i\in\mathbb{R}^d$ and $\vec{y}_j\in\mathbb{R}^d$ are both $d$-dimensional feature vectors. Assume an original OT problem $W_1(\mathcal{D}_x,\mathcal{D}_y,d_2)$, which is converted into an ASOT problem as $W_{\rm AS}(\mathcal{D}_x,\mathcal{D}_y,\mathcal{W},P_o,d_2)$, where $d_2$ is the Euclidean distance function. Then the following inequality holds:
\begin{eqnarray*}
|W_{\rm AS}(\mathcal{D}_x, \mathcal{D}_y, \mathcal{W}, P_o, d_2) - W_1(\mathcal{D}_x, \mathcal{D}_y,d_2)| \\
\leq m\sum_{{i\in\llbracket n\rrbracket}} \|\vec{\epsilon}_i^x\|_2 + n\sum_{{j\in\llbracket m\rrbracket}}\|\vec{\epsilon}_j^y\|_2,
\end{eqnarray*}
where $\vec{\epsilon}_i^x:= \mat{W}P_o(\vec{x}_i) - \vec{x}_i, \vec{\epsilon}_j^y := \mat{W}P_o(\vec{y}_j) - \vec{y}_j$ represent the offsets of reconstructed feature vectors to the original ones.
\end{proposition}

\begin{proof}
{
According to {\bf Proposition \ref{Prop.UB}}, we have $|W_{\rm AS}(\mathcal{D}_x, \mathcal{D}_y, \mathcal{W}, P_o, d_2) - W_1(\mathcal{D}_x, \mathcal{D}_y,d_2)| \leq \| \widehat{\mat{C}} - \mat{C} \|$. Expanding this inequality obtains
\begin{eqnarray*}
	\eta&:=&|W_{\rm AS}(\mathcal{D}_x, \mathcal{D}_y, \mathcal{W}, P_o, d_2) - W_1(\mathcal{D}_x, \mathcal{D}_y,d_2)|\\ 
	&\leq& \| \widehat{\mat{C}} - \mat{C} \|_1\\
	&=&\sum_{{i\in\llbracket n\rrbracket,j\in\llbracket m\rrbracket}}\left|\widehat{\mat{C}}(i,j) - \mat{C}(i,j)   \right|.
\end{eqnarray*}

Because $\widehat{\mat{C}}$ can be decomposed into $\widehat{\mat{C}} = \mat{Z}_x\mat{C}_s\mat{Z}^T_y$ according to {\bf Proposition 1}, we have
\begin{eqnarray*}
	\eta&\leq& \sum_{{i\in\llbracket n\rrbracket,j\in\llbracket m\rrbracket}}\left|\sum_{{u,v\in\llbracket k\rrbracket}}\vec{z}^x_i(u) \vec{z}^y_j(v)\mat{C}_s(u,v)- \mat{C}(i,j)   \right|,
\end{eqnarray*}
where $\vec{z}^x_i:=P_o(\vec{x}_i), \vec{z}^y_j:=P_o(\vec{y}_j)$.
Because $P_o$ is a one-hot mapping function, we can let the index of the non-zero element of $P_o(\vec{x}_i)$ be $u_i$, and let the one of $P_o(\vec{x}_i)$ be $v_i$. Substituting this into the equation above obtains
$
	\eta\leq\sum_{{i\in\llbracket n\rrbracket,j\in\llbracket m\rrbracket}}\left| \mat{C}_s(u_i,v_j) - \mat{C}(i,j) \right|.
$
From {\bf Definition 3}, we have $\mat{C}_s(u_i,v_j) := d_2(\vec{w}_{u_i}, \vec{w}_{v_j}) := \|\vec{w}_{u_i}- \vec{w}_{v_j}\|_2$ and $\mat{C}(i,j) := d_2(\vec{x}_i, \vec{y}_j) := \|\vec{x}_i - \vec{y}_j\|_2$. Then, $\epsilon_i^x = \vec{w}_{u_i} - \vec{x}_i,  \epsilon_j^y = \vec{w}_{v_j} - \vec{y}_j$. Substituting these into the inequality above obtains
\begin{equation*}
	\!\!\eta\leq\!\!\sum_{{i\in\llbracket n\rrbracket,j\in\llbracket m\rrbracket}}\left| \|\vec{w}_{u_i}- \vec{w}_{v_j}\|_2 - \|\vec{w}_{u_i} - \vec{w}_{v_j} + \vec{\epsilon}_j^y - \vec{\epsilon}_i^x\|_2 \right|\!\!.
\end{equation*}

According to the reverse triangle inequality \cite{1855popular}, we have
\begin{eqnarray*}
	\eta&\leq&\sum_{{i\in\llbracket n\rrbracket,j\in\llbracket m\rrbracket}}\|\vec{\epsilon}_j^y - \vec{\epsilon}_i^x\|_2\\
	&\leq&\sum_{{i\in\llbracket n\rrbracket,j\in\llbracket m\rrbracket}}\|\vec{\epsilon}_j^y\|_2 + \|\vec{\epsilon}_i^x\|_2\\
	& =&m\sum_{{i\in\llbracket n\rrbracket}} \|\vec{\epsilon}_i^x\|_2 + n\sum_{{j\in\llbracket m\rrbracket}}\|\vec{\epsilon}_j^y\|_2,
\end{eqnarray*}
which completes the proof.
}
\end{proof}

From {\bf Proposition \ref{Prop.UBoh}}, with a special one-hot encoding $P_o$ and a shared $\ell_2$-norm distance function $d_2$ for both the anchor space and original space, the distance error can be upper bounded within a summation of reconstruction errors. In this case, only two parameters are to be learned: $\mathcal{W}$ and $P_o$. Both our proposed ASOT-$k$ and ASOT-DL methods are based on the same optimization problem derived from {\bf Proposition \ref{Prop.UBoh}}, which is  
\begin{equation}
\mathop{\rm minimize}_{\mathcal{W},P_o}~\sum_{{i\in\llbracket n\rrbracket}}\|\vec{\epsilon}_i^x\|_2 + \sum_{{j\in\llbracket m\rrbracket}}\|\vec{\epsilon}_j^y\|_2.\label{Eq.prop2}
\end{equation}

This objective function only needs computation of $O(n+m)$ complexity. Furthermore, for limitation of the unsafe solution, because of the one-hot constraint and the share of the metric function $d_{\rm S}$ over both the anchor space and the original space, the learned cost matrix of anchor points $\mat{C}_s$ usually has the same scale as that of the original OT problem $\mat{C}$. Therefore, it is less likely to provide unsafe solutions with {\bf Proposition \ref{Prop.UBoh}}. Finally, because {\bf Proposition \ref{Prop.UBoh}} is built upon extra conditions from {\bf Proposition \ref{Prop.UB}}, its solution set should be included in the latter one. Theoretically, ASOT-ML can provide better solutions than either ASOT-$k$ or ASOT-DL.

By minimizing reconstruction errors, the problem becomes a classical clustering problem or a dictionary learning problem. A clustering problem can be solved easily with the $k$-means algorithm and its variants, where $\mathcal{W}$ is the set of cluster centers and $P_o$ is the clustering operator. Therefore, we use the $k$-means algorithm to learn the anchor space in our ASOT-$k$ method. However, the one-hot condition might be too strict, which makes it less possible to achieve low reconstruction errors. Therefore, we also propose a deep dictionary learning framework ASOT-DL to solve a relaxed problem of Eq. (\ref{Eq.prop2}). More specifically, we use an $\ell_2$-ball constraint and a simplex constraint to represent the one-hot condition and then propose a relaxed problem based on them. We adopt a framework from \cite{scetbon2021deep}, named deep K-SVD, with the addition of some modifications to solve this relaxed problem.

\paratitle{Total loss for deep dictionary learning.} We first elaborate on the definition of total loss for our deep dictionary learning framework before presenting details of the network structure.
Let $\vec{x}_i\in\mathbb{R}^d$ be a $d$-dimensional sampling point $\mat{X}:= [\vec{x}_1, \vec{x}_2, \cdots, \vec{x}_n]^T\in\mathbb{R}^{n\times d}$.
Let $\vec{z}_i := \widetilde{P}_o(\vec{x}_i) \in\mathbb{R}^k$ be a near-one-hot encoding for $\vec{x}_i$, $\mat{Z} := [\vec{z}_1, \cdots, \vec{z}_n]^T \in\mathbb{R}^{n\times k}$. Given a anchor point matrix $\mat{W} :=  [\vec{w}_1,\cdots,\vec{w}_k] \in\mathbb{R}^{d\times k}$, we first define a total loss $\mathcal{L}$ for our deep dictionary learning framework as $\mathcal{L}(\mat{W}, \mat{Z}) := \alpha  \cdot \mathcal{L}_{\rm rc} + \beta \cdot \mathcal{L}_{\rm l2}+ \gamma \cdot \mathcal{L}_{\rm sp}$, 
where 
$\mathcal{L}_{\rm rc}$,
$\mathcal{L}_{\rm l2}$, and 
$\mathcal{L}_{\rm sp}$ respectively represent a reconstruction loss, an $ \ell_2$-ball loss, and a simplex constraint violation loss. Also, $\alpha, \beta, \gamma \in\mathbb{R}_+$ denote the weights of their corresponding losses. These losses are computed as
$\mathcal{L}_{\rm rc} := \frac1{2n}\sum_{i = 1}^n\|\mat{W}\cdot(\mat{Z}(i,:))^T - (\mat{X}(i,:))^T\|_2^2,
\mathcal{L}_{\rm l2} := \frac1n\sum_{i=1}^n\left(\|\mat{Z}(i,:)\|_p - 1\right)^2,
\mathcal{L}_{\rm sp} := \frac1n\sum_{i=1}^n\left(\mat{Z}(i,:)\vec{1}_k - 1\right)^2,$
where $p > 1$ is a scalar. For the $\ell_2$-ball loss and the simplex constraint violation loss, they penalize the learned near-one-hot codings to constrain them within an intersection of the surface of a unit $\ell_2$-ball and a unit simplex $\Delta^u_k$, which is the one-hot vector set $\{\vec{q}_k^i\}_{i=1}^k$. Therefore, we can optimize an approximate solution by solving the following relaxed problem of Eq. (\ref{Eq.prop2}): $\mathop{\rm minimize}_{\bm{W},\bm{Z}}l(\mat{W},\mat{Z})$.

\begin{figure}
	\centering
	\includegraphics[width= 0.8\hsize]{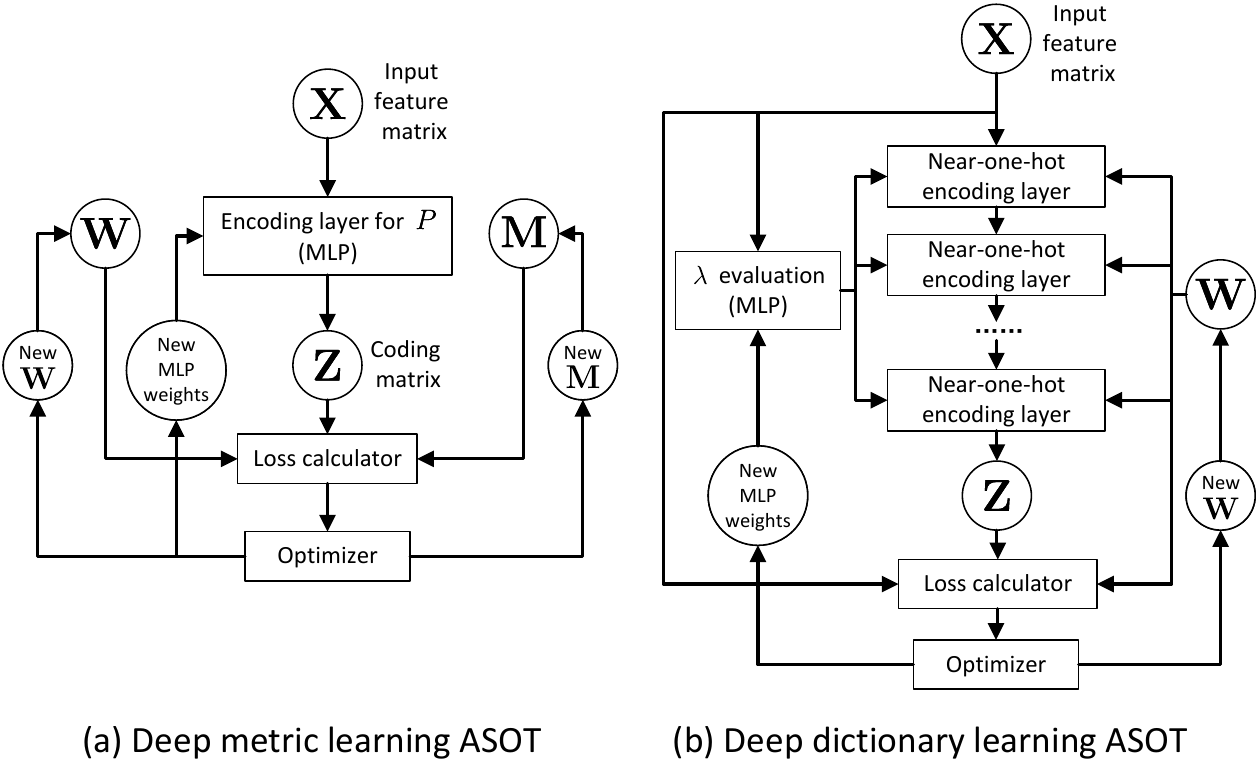}

	\caption{The overall framework of our proposed deep anchor space learning methods, where $\mat{X}\in\mathbb{R}^{n\times d}$ denotes the input feature matrix of sampling points, $\mat{Z}:= P(\mat{X})$ denotes the matrix of anchor space coordinates, $\mat{W}\in\mathbb{R}^{k\times d}$ represents the anchor point matrix and $\mat{M}\in\mathbb{R}^{h\times k}$ stands for the parameter matrix for metric learning. The mapping function $P$ of ASOT-ML from (a) is parameterized with an MLP with output normalization, whereas $\widetilde{P}_o$ of ASOT-DL from (b) consists of several near-one-hot encoding layers. Each computes one iteration of Eq. (\ref{Eq.sp_layer}) with an adapted $\lambda$ evaluated by MLP.}
	\label{Fig.DicLearn}
\end{figure}

\paratitle{Framework structure details} 
The original paper of deep K-SVD presented a deep dictionary learning model based on the ISTA algorithm \cite{daubechies2004iterative} with an adapted $\lambda$, which is actually a sparse encoder. The ISTA algorithm solves the lasso problem $\mathop{\rm minimize}_{\bm{z}}\frac12\|\mat{W}\vec{z} - \vec{x}\|_2^2 + \lambda\|\vec{z}\|_1$ with the proximal gradient descent method which projects the solution of each iteration into an $\ell_1$-ball by adopting the soft threshold function. However, because our {\bf Definition \ref{Def.ASOT}} requires a simplex mapping function $P$, we shrink the solution domain from an $\ell_1$-ball to a unit simplex $\Delta^u_k$ to remove unexpected solutions for faster computation. This process leads to the following problem of
\begin{equation}
	\label{Eq.new_lasso}
	\mathop{\rm minimize}_{{\bm{z}\in\mathbb{R}^k_+}}~\frac12\|\mat{W}\vec{z} - \vec{x}\|_2^2 + \lambda(\vec{z}^T \vec{1}_k - 1),
\end{equation}
where $\lambda \in \mathbb{R}_{++}$. This problem is solvable by the projected gradient descent with the orthogonal projection operator as
\begin{equation}
	\label{Eq.sp_layer}
	\vec{z}  \leftarrow  S_{\lambda}\left(\vec{z} - \mat{W}^T(\mat{W}\vec{z} - \vec{x})\right){,}
\end{equation}
where $S_{\theta}(\vec{a}): \mathbb{R}^k \to\mathbb{R}_{+}^k$ for a vector $\vec{a} \in \mathbb{R}^k$ is the orthogonal projection operator onto the unit simplex, which is defined as $S_{\theta}(\vec{a}):= {\rm max}(\vec{0}, \vec{a} - \theta)$. We present a derivation of this operator in the next paragraph. Figure \ref{Fig.DicLearn}(b) presents an overview of our framework. In that framework, Eq. (\ref{Eq.sp_layer}) is represented as a near-one-hot encoding layer. Our {near-one-hot encoder for $\widetilde{P}_o$} consists of $L$ sequentially connected near-one-hot encoding layers to generate $\vec{z}$, where $L$ denotes the number of layers. We also include a multilayer perception network (MLP) to evaluate $\lambda$ for each input as deep K-SVD does.

\paratitle{Derivation of Eq. (\ref{Eq.sp_layer})}. {Recall the primal problem that is expressed as
\begin{eqnarray}
\mathop{\rm minimize}_{\bm{z}\in\mathbb{R}^k_+,\langle\bm{z},\bm{1}_k\rangle \leq 1} \left\{f(\vec{z}):=\frac12\|\mat{W}\vec{z} - \vec{x}\|_2^2 \right\}\label{Eq.primal},
\end{eqnarray}
where we want to optimize a \vec{z} on the unit simplex $\Delta^u_k$ that minimizes the squared reconstruction error. We first ignore the non-negative constraint $\vec{z}\in\mathbb{R}^k_+$, and we compute the Lagrangian of $f(\vec{z})$ in Eq.~(\ref{Eq.primal}) as follows.
$
L(\vec{z},\lambda) := \frac12\|\mat{W}\vec{z} - \vec{x}\|_2^2 + \lambda(\langle\vec{z},\vec{1}_k\rangle - 1).
$
Assuming that the value of $\lambda > 0$ is already known, the minimization of the Lagrangian with respect to $\vec{z}\in\mathbb{R}^k_+$ can be expressed as follows.
\begin{eqnarray}
\mathop{\rm minimize}_{\bm{z}\in\mathbb{R}^k_+}~\frac12\|\mat{W}\vec{z} - \vec{x}\|_2^2 + \langle\vec{z},\lambda\vec{1}_k\rangle\label{Eq.lagrangian}.
\end{eqnarray}

Let 
$g(\vec{z}) =  \langle\vec{z},\lambda\vec{1}_k\rangle$, according to the proximal gradient descent method \cite{parikh2014proximal}, the update equation of $\vec{z}$ with a fixed step size as $1$ is expressed as follows.
\begin{eqnarray}
	\vec{z}^{(t+1)} &=& {\rm prox}_g(\vec{z}^{(t)} - \nabla f(\vec{z}^{(t)}))\cr
	&=&{\rm prox}_g(\vec{z}^{(t)} - \mat{W}^T(\mat{W}\vec{z}^{(t)} - \vec{x}))
	\label{Eq.prox},
\end{eqnarray}
where ${\rm prox}_g$ denotes the proximal operator as
\begin{eqnarray}
{\rm prox}_g(\vec{a}) &\!\!:=\!\!& \mathop{\rm arg~min}_{\bm{a}} g(\vec{a}) + \frac12\|\vec{z} - \vec{a}\|_2^2\cr
&\!\!=\!\!& \mathop{\rm arg~min}_{\bm{a}} \frac12\|\vec{z} - \vec{a}\|_2^2 + \langle\vec{a},\lambda\vec{1}_k\rangle\cr
&\!\!=\!\!& \mathop{\rm arg~min}_{\bm{a}} \sum_{i = 1}^k\left(\frac12(\vec{z}(i) - \vec{a}(i))^2 + \lambda\vec{a}(i)\right).\quad \quad\label{Eq.prox_subproblem}
\end{eqnarray}

Eq.~(\ref{Eq.prox_subproblem}) can be written as $k$ subproblems w.r.t. each element of $\vec{a}$, such as
$
{\rm prox}_g(\vec{a})(i) = \mathop{\rm arg~min}_{\bm{a}(i)}\frac12(\vec{z}(i) - \vec{a}(i))^2 + \lambda\vec{a}(i).
$
Because of the non-negative condition of $\vec{z}$ and Eq.~(\ref{Eq.prox}), we have $\forall {i \in\llbracket k\rrbracket}, \vec{a}(i)>0$. Therefore, the solutions of the above subproblems are $\vec{a}(i) = \max(0, \vec{z}(i) - \lambda)$. Generalizing these solutions to the vector $\vec{a}$ yields
\begin{eqnarray}
{\rm prox}_g(\vec{a}) = \max(\vec{0}, \vec{a} - \lambda\vec{1}_k)\label{Eq.proj_op},
\end{eqnarray}
which derives the equation of our orthogonal projection operator $S_\theta(\vec{a})$ of the paper. Substituting Eq.~(\ref{Eq.proj_op}) into Eq.~(\ref{Eq.prox}) yields Eq. (\ref{Eq.sp_layer}),
which completes the derivation.
}

\section{Numerical Evaluation}
\label{Sec.exp}

Numerical experiments are conducted through three aspects: (i) distance approximation evaluation, (ii) computational time evaluation, and (iii) ablation studies of the main parameter $k$. This section presents reports of our experimentation.

\begin{figure*}
\centering
\includegraphics[width= \hsize]{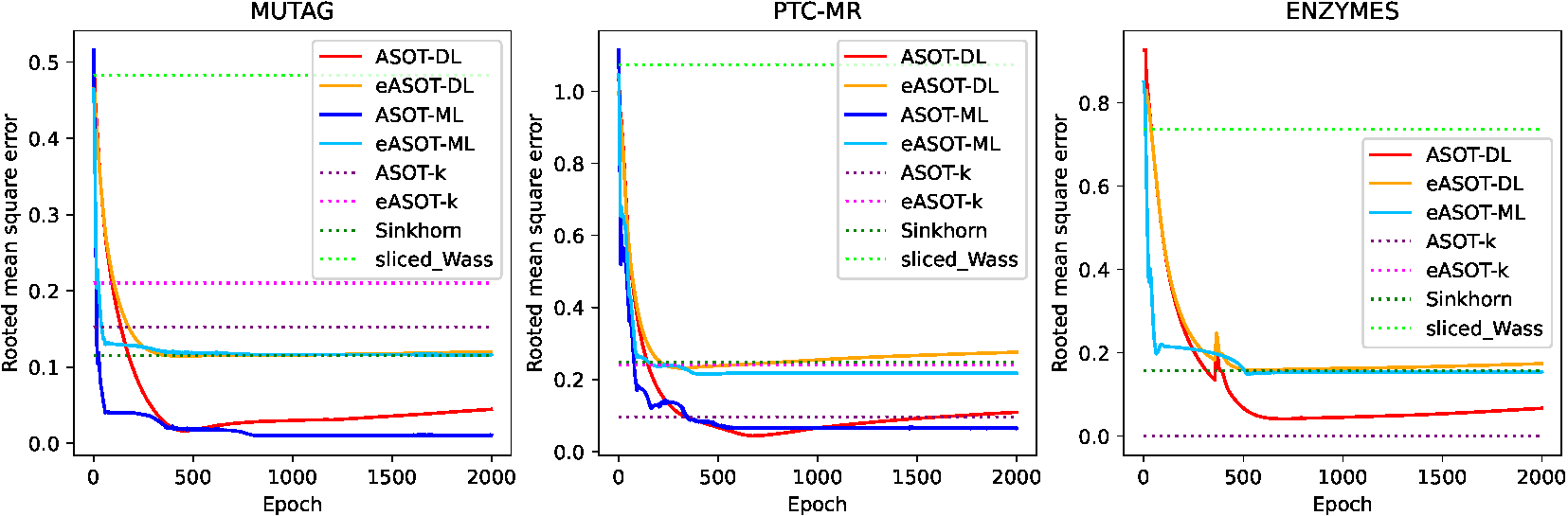}
\caption{Average approximation error curves for graph datasets across training epochs.}
\label{Fig.avg_approx}
\end{figure*}

\begin{figure}
\centering
	\includegraphics[width= 0.65\hsize]{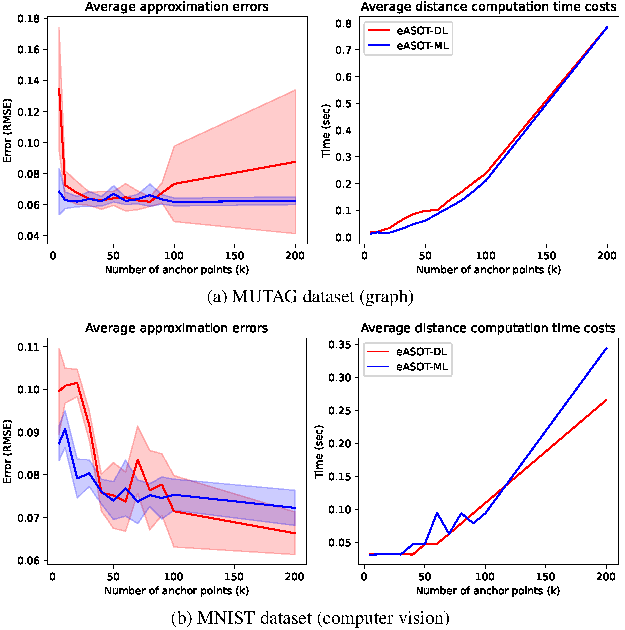}
	\label{Fig.ml_model}

	\caption{Average approximation error and computational time cost curves of an ablation study of parameter $k$, where this report describes the results of a series of $k$ of $\{5, 10, 20, 30, 40, 50, 60, 70, 80, 90, 100, 200\}$.}
	\label{Fig.ablation}
\end{figure}

\begin{table*}
\caption{Average approximation errors of graph datasets (RMSE).}
\label{Tab.err}
\begin{center}
\scalebox{0.7}{
\begin{tabular}{lcccccccc}
\toprule
Methods&MUTAG&PTC-MR&ENZYMES&BZR&COX2&NCI1&ZINC&MNIST\\
\midrule
SW (CPU)&0.4826&1.0729&0.3253&0.3305&0.3253&0.6061&0.5761&0.1084\cr
eOT (CPU)&0.1155&0.2489&0.1556&0.1366&0.0827&0.0791&0.0754&0.0981\cr
BDS-eOT (GPU)&0.1035&0.0675&0.1487&0.1359&0.0807&0.0771&0.0744&0.0964\cr
\midrule
ASOT-ML (CPU)&0.0264&0.0638&0.0021&0.0117&0.0087&0.0214&0.0297&0.0869\cr
ASOT-DL (CPU)&0.0262&0.1324&0.0854&0.0227&0.0277&0.0453&0.0980&0.0972\cr
ASOT-$k$ (CPU)&0.1169&0.1562&0.0195&0.0286&0.0282&0.0831&0.0460&0.0157\cr
{\bf eASOT-ML (GPU)}&0.1216&0.2243&0.1544&0.1363&0.0832&0.0840&0.0811&0.0778\cr
{\bf eASOT-DL (GPU)}&0.1173&0.2897&0.1710&0.1376&0.0941&0.0902&0.0889&0.0990\cr
{\bf eASOT-$k$ (GPU)}&0.1821&0.2358&0.1657&0.1476&0.0975&0.1243&0.0988&0.0802\cr
\bottomrule
\end{tabular}}
\end{center}

\caption{Average computational time cost of graph datasets (sec). For ASOT methods, the value at the left of ``+'' denotes the average training time, whereas the one on the right denotes the distance computing time.}
\label{Tab.time}
\begin{center}
\scalebox{0.7}{
\begin{tabular}{lcccccccc}
\toprule
Methods&MUTAG&PTC-MR&ENZYMES&BZR&COX2&NCI1&ZINC&MNIST\\
\midrule
OT-EMD (CPU)& 2.14&7.68&28.89&30.51&29.30&1701&2708&199.9\cr
SW (CPU)&2.59&13.45&23.53&91.45&62.05&4512&5190&25.92\cr
eOT (CPU)&40.73&123.1&392.2&173.1&242.1&18574&29090&572.5\cr
BDS-eOT (GPU)&6.78&26.34&85.36&42.48&56.42&3970&5703&131\cr
\midrule
ASOT-ML (CPU)&1.84+2.00&1.92+6.22&12.8+27.77&7.38+13.53&12.27+20.50&64.6+1290&378+1509&0.41+126\cr
ASOT-DL (CPU)&5.76+2.06&4.70+9.95&4.76+25.27&5.74+16.76&5.88+27.35&8.04+1470&285+2151&0.82+9.59\cr
ASOT-$k$ (CPU)&0.1+1.97&0.09+6.68&0.15+20.92&0.17+9.51&0.20+13.00&0.73+994&37+2046&0.008+52.50\cr
{\bf eASOT-ML (GPU)}&1.84+0.01&1.92+0.03&12.8+0.18&7.38+0.10&12.27+0.17&64.6+6.92&378+6.95&0.41+2.47\cr
{\bf eASOT-DL (GPU)}&5.76+0.04&4.70+0.07&4.76+0.24&5.74+0.17&5.88+0.32&8.04+9.88&285+8.91&0.82+2.97\cr
{\bf eASOT-$k$ (GPU)}&{\bf 0.1+0.06}&{\bf 0.09+0.07}&{\bf 0.15+0.25}&{\bf 0.17+0.18}&{\bf 0.20+0.26}&{\bf 0.73+7.42}&{\bf 37+9.16}&{\bf 0.008+2.42}\cr
\bottomrule
\end{tabular}}
\end{center}
\end{table*}
\subsection{Experiment setup}
\label{Sec.setup}

\paratitle{Datasets.} This report describes our experiments conducted using seven real-world graph datasets: MUTAG \cite{debnath1991structure}, PTC-MR \cite{helma2001predictive}, ENZYMES \cite{schomburg2004brenda}, BZR, COX2 \cite{sutherland2003spline}, NCI1\cite{wale2008comparison}, and ZINC \cite{bresson2019two} from the TUdataset \cite{Morris+2020}. The latter, ZINC, is a large-scale dataset containing $249456$ graphs, which is much larger than the other graph datasets used for our experimentation. Therefore, our experiment setup for ZINC is slightly different from the others. We describe the experiment setups in the ``Baselines and other setups'' paragraph. Furthermore, we report experiments on a computer vision dataset: MNIST \cite{lecun1998mnist}. The MNIST database includes two splits: train and test subsets, where the {training} subset has $60000$ images of handwritten digits from $0$ to $9$, whereas the test subset has $10000$ images. Each image is resized into a $28 \times 28$ matrix.

\paratitle{Preprocessing on graph datasets.} Because the node features of the datasets are one-hot vectors of discrete labels, which are too simple for evaluating approximation performance, we conduct a feature updating preprocessing before training. Thereby, we can generate more complicated features. More concretely, we conduct a classic message-passing-based updating equation as follows to update node features: $\vec{x}_i^{\prime} := (1+\epsilon) \cdot \vec{x}_i + \sum_{j\in\mathcal{N}(i)} \vec{x}_j$, where $\mathcal{N}(i)$ denotes the set of indices of adjacent node of {the} $i$-th node, $\vec{x}_i\in\mathbb{R}^d$ denotes the feature of {the} $i$-th node, and $\vec{x}_i^\prime\in\mathbb{R}^d$ denotes its updated feature. This is exactly the updating equation of the famous graph isomorphism network \cite{xu2018powerful} without the MLP layer. Specifically, we set $\epsilon = 0$. 
For all data, we compute four updates of equations iteratively. The final feature for each node is the concatenation of its original feature and outputs of each iteration as $[\vec{x}_i, \vec{x}_i^{\prime}, \vec{x}_i^{\prime\prime}, \vec{x}_i^{\prime\prime\prime}, \vec{x}_i^{\prime\prime\prime\prime}]\in\mathbb{R}^{5d}$.

\paratitle{Preprocessing on MNIST dataset.} We follow the experiments in Dvurechensky's work \cite{dvurechensky2018computational} to solve the OT problem on {the} MNIST dataset, where they treat images of handwritten digits as one-dimensional discrete distributions and compute their {Wasserstein} distance. However, the discrete distributions used in their experiments are the same size (equal to the number of pixels of a single image). To demonstrate the performance of our proposed method on size-variable distributions, we transform images into distributions of different sizes by filtering out their zero elements. During preprocessing, we only sample the pixels with non-zero values from the image. More concretely, assume ${i,j \in \llbracket 28\rrbracket}$. Let ${p(i,j)\in\llbracket 256\rrbracket}$ be the pixel value of $i$-th row and $j$-th column. Image data are converted to a discrete distribution with three-dimensional sampling points as $\{[\frac{i}{28},\frac{j}{28}, \frac{p(i,j)}{256}]\in\mathbb{R}^3:{i,j\in\llbracket 28\rrbracket}, p(i,j)>0\}$. Because we only sample pixels with values larger than $0$, the sizes of these distributions vary.

\paratitle{Baselines and other setups.} We set the Earth Movers Distance (OT-EMD) solved by the EMD solver \cite{bonneel2011displacement} as the baseline distance, which computes an exact solution of the OT problem in Eq. (\ref{Eq.OT}) in the form of a linear programming problem. For our proposed methods, in addition to ASOT-ML and ASOT-DL, we also implement a variant ASOT-$k$, for which the anchor space is learned by $k$-means according to {\bf Proposition \ref{Prop.UBoh}}. eASOT-ML, eASOT-DL, and eASOT-$k$ denote the entropic ASOT solved using Sinkhorn's solver with GPU-parallelization, respectively for ASOT-ML, ASOT-DL and the $k$-means variant. ASOT-ML, ASOT-DL, and ASOT-$k$ denote the ones with the EMD solver. For other compared approximate methods, we choose the following: The entropic OT with the {Sinkhorn's algorithm} \cite{Peyre_2019_OTBook_s} on {CPU (eOT), the entropic OT with the {Sinkhorn's algorithm} and the block diagonal stacking on GPU (BDS-eOT)}, and the sliced Wasserstein distance (SW) \cite{bonneel2015sliced}. For the parameter settings of other methods, we use Sinkhorn's solver with $\varepsilon = 0.1$ and $50$ iterations, a sliced Wasserstein distance solver with $d$ projections. For our methods, we set the number of anchor points $k$ equal to the average number of nodes of the dataset contained by a graph. For the weights of losses of ASOT-DL, we set $\alpha = 1, \beta=1$, and $\gamma=1$. For ASOT-ML, we set $C=100$. We conduct $500$ epochs of training with a batch size of $500$ (graphs), except for ZINC, for which we only compute 50 epochs. For the subset division of graph datasets, we run a $10$-fold cross-validation where train vs. test is $9:1$, except for ZINC, where we use the default splits. We compute the pairwise distance of the whole dataset (train + test) for evaluation, except for ZINC, for which we only compute the test subset. For the subset division of the MNIST dataset, we randomly choose $50$ images for each category of digits from the test subset, with which we construct a subset of $500$ images. We conduct $10$ times distance approximation experiments on these subsets, which are generated randomly with different random seeds for each run. The experimental setup used a machine equipped with an Intel i7-12700KF CPU, 32 GB of RAM, and an Nvidia GeForce RTX 3090Ti GPU.

\subsection{Wasserstein distance approximation}

\paratitle{Distance approximation evaluation.} For these experiments, we simulate a general task of pairwise Wasserstein distance matrix computation, where the distance between two graphs (images) is defined as the $1$-Wasserstein distance with Euclidean ground metric between their node feature sets {(pixel feature sets)}. For the evaluation criterion, we adopt the rooted mean square error (RMSE) between approximated distance matrices and the ground truth computed using EMD. Figures \ref{Fig.avg_approx} show the curves of approximation errors of three graph datasets, where it is apparent that the approximation errors decrease as we minimize the losses of ASOT-ML and ASOT-DL. TABLE \ref{Tab.err} presents the average approximation errors. We omit the standard deviations here because all of them are less than $0.05$. In this table, we compare the ASOT problems solved by the EMD solver (ASOT-ML, ASOT-DL, ASOT-$k$) with the baseline method OT-EMD and the entropic ASOT problem solved using the Sinkhorn solver (eASOT -ML, eASOT-DL, eASOT-$k$) versus eOT and BDS-eOT because they are solved respectively using the same solver. In summary, the approximation results of ASOT-ML, ASOT-DL and ASOT-$k$ are very close to the baseline method. The approximation errors of eASOT-ML, eASOT-DL, and eASOT-$k$ are also very close to those of eOT and BDS-eOT. Among them, ASOT-ML and eASOT-ML exhibit the best approximation performance.

\paratitle{Computational time cost evaluation.} We also evaluate the distance computation time cost in the distance approximation experiments. TABLE \ref{Tab.time} presents records of the time costs (divided into model training time and distance computing time) in seconds, where the smallest value is shown with bold typeface. For CPU-based methods, our proposed methods using the EMD solver (ASOT-ML, ASOT-DL, ASOT-$k$) show better distance computation in most datasets compared to baseline methods (OT-EMD) in time. For large-scale datasets such as NCI1, ZINC, and MNIST, the CPU-based ASOTs also show great time reductions even if the model-training time is added to the distance-computing time. For GPU-based methods, our proposed methods using Sinkhorn's solver (eASOT-ML, eASOT-DL, eASOT-$k$) show remarkable time reduction (approximately 50-500 times faster distance computation) compared to the BDS-eOT. In summary, the $k$-means variants (ASOT-$k$, eASOT-$k$) show the best cost performance because their training time and distance computing time are both very fast. Although ASOT-ML and eASOT-ML have the best approximation performance, their training time is the slowest among our proposed methods because they require calculating the cost matrix in the training subset.  

\subsection{Ablation studies}
As described earlier, parameter $k$ is considered the most important parameter of our proposed ASOT. It represents the number of learnable anchor points. To study how the approximation error deteriorates and computational time decreases with different $k$, this report describes two ablation study experiments using two datasets: the graph dataset MUTAG and the computer vision dataset MNIST. We prepare a candidate set of $k$ as ${\{5, 10, 20, 30, 40, 50, 60, 70, 80, 90, 100, 200\}}$. For each dataset, we conduct {five} runs of a single experiment and then report the means and standard deviations of the approximation error. We keep the experiment setups the same as described in Section \ref{Sec.setup}, except for the number of the Sinkhorn iterations, which we set as $200$ for MUTAG and as $500$ for MNIST. Figure \ref{Fig.ablation} show{s} the experimentally obtained results of approximation error and computational time with different $k$. From the curves, one can observe a stable descent of the approximation error curves of eASOT-ML on both datasets, whereas that of eASOT-DL shows signs of rebound on the MUTAG dataset when $k$ is larger than 90. We believe that this rebound occurs because {eASOT-DL optimizes a relaxed solution of the upper bound in {\bf Proposition \ref{Prop.UBoh}}}, which makes it easier to reach the limits and therefore overfit.

\section{Conclusion}
{This paper proposes {an approximate} OT problem, named the ASOT problem, over the anchor space for multiple OT problems. For the proposed ASOT, we restrict the mass transport to a learned and fixed anchor space, which is a subset of the original feature space.} Then, we prove the upper bound of the Wasserstein distance approximation error of such an ASOT problem. Based on these findings, we propose metric learning (ASOT-ML), a $k$-means-based (ASOT-$k$), and a deep dictionary learning framework (ASOT-DL) for anchor space learning. Experimentally obtained results demonstrate that our methods decrease computational time significantly while maintaining a reasonable level of approximation performance. For future work, we aim to design closed-form expressions for the mapping and metric function of ASOT, thereby eliminating the requirement for anchor space learning.

\bibliographystyle{unsrt}
\bibliography{nips, graph, OT, graph_datasets, kasailab}
	
\end{document}